\documentclass{article}

\usepackage{microtype}
\usepackage{graphicx}
\usepackage{booktabs} 

\usepackage{hyperref}

\usepackage{amsmath}
\usepackage{amssymb}
\usepackage{amsthm}

\usepackage{thmtools} 
\usepackage{thm-restate}

\usepackage{parskip} 

\usepackage{pgfplots}
\theoremstyle{plain} 
\newtheorem{theorem}{Theorem}[section]
\newtheorem{corollary}{Corollary}[theorem]
\newtheorem{lemma}[theorem]{Lemma}
\newtheorem{remark}[theorem]{Remark}

\newtheorem{definition}{Definition}[section]
\newenvironment{rcases}
  {\left.\begin{aligned}}
  {\end{aligned}\right\rbrace}

\newcommand{\norm}[1]{\left\lVert#1\right\rVert}

\newcommand{\D}[2]{\operatorname{D}_{#1}{\hspace{-.7mm}#2}}

\renewcommand{\D}[2]{\frac{\operatorname{d} #2}{\operatorname{d} #1}}

\usepackage{graphicx}
\usepackage{import}
\usepackage{verbatim}

\makeatletter
\newcommand*\bigcdot{\mathpalette\bigcdot@{.5}}
\newcommand*\bigcdot@[2]{\mathbin{\vcenter{\hbox{\scalebox{#2}{$\m@th#1\bullet$}}}}}
\makeatother

\def\ph{{\hspace{2pt}\bigcdot\hspace{2pt}}} 

\def\xj{{x_{j}}}

\def\zj{{z_{j}}}

\def\zetaj{{\zeta_{j}}}

\def\gvar{{q}}

\def\hh{{h}}

\def\thetaj{{\theta_{j}}}

\def\L{\mathcal{L}}
\def\R{\mathcal{R}}
\def\P{\mathcal{P}}
\def\X{\mathcal{X}}
\def\Y{\mathcal{Y}}

\def\zetaj{\zeta_{j}}

\def\gradient{{\partial_{x_0}x_L}}

\def\transp{T}
\def\wup{\Box}
\def\ith{{\langle i \rangle}}
\newcommand{\Delete}[1]{\mathbf{delete \hphantom{.}} #1}

\def\globalscale{.5}

\usepackage{bbm}

\usepackage{tikz}
\usepackage{verbatim}

\usepackage[caption=false,font=footnotesize,labelfont=sf,textfont=sf]{subfig}

\usepackage[accepted]{icml2019}

\icmltitlerunning{A Closer Look at Double Backpropagation and other Gradient Penalties}

\begin{document}

\twocolumn[
\icmltitle{A Closer Look at Double Backpropagation}



\icmlsetsymbol{equal}{*}

\begin{icmlauthorlist}
\icmlauthor{Christian Etmann}{hb}
\end{icmlauthorlist}

\icmlaffiliation{hb}{Center for Industrial Mathematics, University of Bremen, Bremen, Germany}

\icmlcorrespondingauthor{Christian Etmann}{cetmann@uni-bremen.de}

\icmlkeywords{Double Backpropagation, Optimization, Loss Landscapes}

\vskip 0.3in
]



\printAffiliationsAndNotice{}  

\begin{abstract}

In recent years, an increasing number of neural network models have included derivatives with respect to inputs in their loss functions, resulting in so-called double backpropagation for first-order optimization. However, so far no general description of the involved derivatives exists. Here, we cover a wide array of special cases in a very general Hilbert space framework, which allows us to provide optimized backpropagation rules for many real-world scenarios. This includes the reduction of calculations for Frobenius-norm-penalties on Jacobians by roughly a third for locally linear activation functions. Furthermore, we provide a description of the discontinuous loss surface of ReLU networks both in the inputs and the parameters and demonstrate why the discontinuities do not pose a big problem in reality.
\end{abstract}

\section{Introduction}
Lately, an increasing number of papers have suggested using penalty terms involving derivatives with respect to the neural network input. So far, no valid and general description of the backpropagation procedure for these cases exists. While \citep{drucker1992improving} derive the double backpropagation formulas for a multilayer perceptron with one hidden layer only, \citep{sokolic2017robust} provide only a high-level description for ReLU \citep{nair2010rectified} networks. While automatic differentiation (AD) methods have made the calculation of the error terms and their respective weight gradients trivial to implement, they do not lend themselves to providing any theoretical insights. However, as we will show here, the specific choice of architecture and activation function can have a large impact on the optimization, for which a precise understanding of the involved backpropagation is essential.  Furthermore, as we will show here, one can improve both the training time and memory requirements of the involved training procedures over the na\"{i}ve utilization of AD in many real-world scenarios.\newline
While it is straightforward to derive the backpropagation terms of neural networks which do not encompass derivate-based regularization terms, the situation looks very different when these are included. This stems from an intricate interdepence of the various involved terms, which needs to be accounted for.

\subsection{Contributions}
We derive backpropagation rules for large classes of derivative-based regularization terms in the very general framework of Hilbert spaces, which covers everything from standard neural networks up to esoteric networks in function spaces along the lines of \citep{bruna2013invariant,wiatowski2017mathematical}. We thereby offer a different perspective on backpropagation, which is usually understood as an operation on a computational graph.\newline 
In neural network literature, the derivatives are most often given in coordinate-form for specific examples of layers, e.g. fully-connected layers. The coordinate-free view in Hilbert spaces offers a unifying view using Fr\'{e}chet derivatives, that is readily applicable to a wide range of problems.
For this, we view the linear portion of e.g. fully-connected, convolutional and locally-connected layers as specific instances of continuous, bilinear operations between the parameters and the activations and extend the standard theory of adjoints of continuous linear operators in real Hilbert spaces to continuous bilinear operators. \newline

We furthermore analyze the runtimes of different variants of double backpropagation and are able to provide adapted algorithms for various scenarios depending on the exact setup, including a reduction by up to a third for certain Jacobian penalties. 
We additionally explore the induced loss landscapes of the common special case of (leaky) ReLU neural networks, which induces jump discontinuities. We demonstrate that batch optimization procedures can alleviate concerns about instabilities caused by these discontinuities. 

\subsection{Applications of Derivative-Based Loss Terms} \label{sec:examples}
Double backpropagation comes into play, whenever one uses derivative-based optimization on loss functions which contain derivatives with respect to the input of the network. There is a variety of applications and model types that employ losses of this type. One example is 'classical' double backpropagation \citep{drucker1992improving}, where the loss for one feature-label-pair $(x,y)$ and forward-mapping $f$ is
\begin{equation}
    \ell (f(x),y) + \lambda \cdot \|\nabla_x \ell (f(x),y)\|_2^2,
\end{equation}
with  loss function $\ell$. One possible application is robustification to adversarial attacks \citep{pmlr-v97-simon-gabriel19a}. Instead of the loss, one may also penalize derivatives of logits or class predictions. In \citep{sokolic2017robust}, the penalty term takes the form $\|J_f\|_F^2,$ the squared Frobenius norm of the Jacobian of the output with respect to the input. Through this penalty term, one can effectively enlargen the model's margin in order to improve its generalization. Another instance of this type of penalty is found in contractive autoencoders \citep{rifai2011contractive}, where the Jacobian is calculated on the encoder's output, which is intended to assign similar codes to similar inputs. If one chooses the spectral norm instead of the Frobenius norm, one idea is to instead use $\|J_f v\|_2^2,$ where $v$ is a random unit vector. This is equivalent to one iteration of the power method, as proposed e.g. in \citep{anil2018sorting}.
For applications where a ground-truth function to be approximated is known (e.g. model compression), Sobolev training \citep{czarnecki2017sobolev} aims to make the model close to the ground-truth in higher-order Sobolev norms, which entails the input's derivatives.\newline
Flow-based generative models like normalizing flows \citep{rezende2015variational}, GLOW \citep{kingma2018glow}, FFJORD \citep{grathwohl2018ffjord} and invertible residual networks \citep{pmlr-v97-behrmann19a} generate a point $x$ via $x=f^{-1}(z)$ by sampling $z$ from a simple base distribution. Here, $f$ can be a neural network. These models seek to maximize the data-likelihood, resulting in a loss function which contains the log-determinant of the Jacobian $J_f$, for whose evaluation various strategies exist.\newline
Another instance of generative models requiring double backpropagation are certain types of generative adversarial networks (GANs) like \citep{roth2017stabilizing}, which enforce convergence through gradient-based penalty terms.\newline
For the solution of inverse problems, adversarial regularizers \citep{lunz2018adversarial} incorporate a critic network whose local Lipschitz constant is kept small via a regularization term which requires double backpropagation.

\subsection{How to Read this Paper}
The utilized framework here are Fr\'{e}chet derivatives on Hilbert spaces, i.e. vector spaces that are complete with respect to the norm $\| \ph \| : u \mapsto \sqrt{\langle u, u \rangle}$ induced by their inner product $\langle \ph , \ph \rangle$. Readers unfamiliar with these terms can still understand most derivations and results by thinking of simple examples. A generic example for a Hilbert space is  $\mathbb{R}^n$ with the standard inner product $\langle u, v \rangle := u^Tv$. Fr\'{e}chet derivatives can then be represented via the well-known concept of a Jacobian matrix. 

\section{Mathematical Preliminaries}

\subsection{Properties of Bilinear Operators}
We introduce continuous, bilinear operators as a very general, yet simple tool for defining the affine linear portion of many different layer types. This encompasses dense, convolutional, locally-connected layers, average pooling and invertible down-sampling. If we take a dense layer as an example, then the pre-activation $Wx+b$ with matrix $W$ and bias $b$ contains an expression that is linear \emph{both} in $x$ and in $W$. We can thus write $K(W,x)=Wx$ and realize that $K$ is a bilinear operator.\newline 
Similarly, for convolutional layers we have $K(w,x)=w \ast x$ with the multi-channel convolution operator $\ast$. A typical example for image data would be $x \in \mathbb{R}^{3 \times 256 \times 256}$ and $w \in \mathbb{R}^{5\times 5 \times 3 \times 16},$ which represents the convolution of 256-by-256 RGB image with a 5-by-5 kernel onto 16 feature maps. The theoretical setting allows us to work directly in these spaces, \emph{without} reordering the entries into column vectors and representing the Fr\'{e}chet derivatives as Jacobians. The proofs for the following theorems are found in Appendix \ref{sec:appdx_bilinear}.\\

In the following, let $\X$, $\Y$ and $\P$ always be real Hilbert spaces. Let $A: \X \rightarrow \Y$ be a continuous linear operator. We denote by $A^\ast$ the adjoint of $A$, i.e. the (unique) linear operator for which $$\langle Ax,y \rangle_\Y = \langle x,A^\ast y \rangle_\X$$ for all $(x,y)\in(\X\times\Y)$, where the $\langle \ph,\ph \rangle$ signify the respective inner products. If $\X=\mathbb{R}^n$ and $\Y=\mathbb{R}^m$, then $A \in \mathbb{R}^{m\times n}$ (up to isomorphism) and its adjoint is just the transposed matrix $A^T$.
We now extend the concept of an adjoint of a linear operator on real Hilbert spaces to bilinear operators and prove some elementary properties.\\

Let 
\begin{equation}\begin{aligned}
K: \P \times \X &\rightarrow \Y \\
  (\theta,x) &\mapsto K(\theta,x)
\end{aligned}\end{equation}
be a bilinear, continuous operator between real Hilbert spaces $\P,\X,\Y$, where  
\begin{equation}\begin{aligned}
K(\theta,\ph): \X &\rightarrow \Y \\
  x &\mapsto K(\theta,x)
\end{aligned}\end{equation}
and 
\begin{equation}\begin{aligned}
K(\ph,x): \P &\rightarrow \Y \\
  \theta &\mapsto K(\theta,x)
\end{aligned}\end{equation}
are continuous linear operators for some fixed values of $\theta$ and $x$ respectively. Let $K^\transp(\theta,\ph)$ be the adjoint of $K(\theta,\ph)$ and $K^\wup(\ph,x)$ be the adjoint of $K(\ph,x)$. These are \emph{linear} operators.
We then define the \emph{bilinear operators}
\begin{equation}\begin{aligned}
K^\transp: \P \times \Y &\rightarrow \X \\
  (\theta,y) &\mapsto K^\transp(\theta,\ph)y
\end{aligned}\end{equation}
and
\begin{equation}\begin{aligned}
K^\wup: \X \times \Y &\rightarrow \theta \\
  (x,y) &\mapsto K^\wup(\ph,y)x,
\end{aligned}\end{equation}
which we call (with some abuse of notation) the \emph{adjoint of $K(\theta,x)$ in $x$} and the \emph{adjoint of $K(\theta,x)$ in $\theta$} respectively. In particular, this means that
$$\langle K(\theta,x),y \rangle_\mathcal{Y} = \langle x, K^\transp(\theta,y) \rangle_\mathcal{X}$$
and
$$\langle K(\theta,x),y \rangle_\mathcal{Y} = \langle \theta, K^\wup(x,y) \rangle_\mathcal{P}.$$ 
\begin{remark}\label{rem:inverse_adjoints}
It follows that $K$ is the adjoint of $K^\transp(\theta,y)$ in $y$ as well as the adjoint of $K^\wup(x,y)$ in $x$.
\end{remark}

\begin{remark}
We call $K^T$ the \emph{transposed operator of $K$}. If $K$ is a convolution, this nomenclature is in accordance with the convention in other publications that call this operator the 'transposed convolution'. We further call $K^\wup$ the \emph{weight-adjoint operator} of $K$. 
\end{remark}
The authors are not aware of a name for the weight-adjoint operator in the literature. In Tensorflow \citep{abadi2016tensorflow} for example, the weight-adjoint of the 2D convolution operator is (perhaps a little vaguely) called 'tf.nn.conv2d\_backprop\_filter' due to its role in backpropagation, as we will discover in section \ref{sec:deriving_double}. 

Since $K^\transp$ and $K^\wup$ are bilinear operators, if they are continuous in both arguments, there exist two adjoint operators for each of them as well. Two of these four operators were already identified in Remark \ref{rem:inverse_adjoints}. We clarify their connection with the following perhaps surprising theorem, which will be essential for the calculation of the double backpropagation rules:\\

\begin{restatable}{theorem}{adjointTheorem}\label{thm:adjoint-commutativity}
Let $K^\transp$ be the adjoint of $K(\theta,x)$ in $x$ and let $K^\wup$ be the adjoint of $K(\theta,x)$ in $\theta$. Let further $K^\transp(\theta,y)$ be continuous in $\theta$ for all $y$ and let $K^\wup(x,y)$ be continuous in $x$ for all $y$. Then the adjoint of $K^\transp(\theta,y)$ in $\theta$ exists and coincides with $K^\wup$ and the adjoint of $K^\wup(x,y)$ in $x$ exists and coincides with $K^\transp$.
\end{restatable}
The above theorem indicates that the weight-adjoint of a transposed operator is the same as the weight-adjoint for its primal operator. While this is easy to see for special cases like matrix-vector multiplication or convolutions, this theorem guarantees this for every bilinear, continuous operator.

\begin{remark}
If all involved spaces are finite, the continuity is guaranteed \citep{schechter1996handbook}.
\end{remark}
\subsection{Fr\'{e}chet Calculus}
Here, we will provide short definitions and theorems for derivatives in Hilbert spaces, which are just generalizations of familiar terms in $\mathbb{R}$. We will always assume the involved spaces to be vector spaces over the field $\mathbb{R}$. The following definitions and theorems are standard and can e.g. be found in even more generality in \citep{schechter1996handbook}.
\begin{definition}
Let $\X$ and $\Y$ be Hilbert spaces and let $U \subset \X$ be an open subset. The function $f : U \rightarrow \Y$ is called \emph{Fr\'{e}chet differentiable in $x\in U$} if there is a continuous linear operator $\D{x}{f(x)}: \X \rightarrow \Y$ such that
$$\lim\limits_{\|h\|_\X\rightarrow 0}\frac{\|f(x+h)-f(x) - \D{x}{f(x)} \cdot h\|_Y}{\|h\|_X}.$$ Then $\D{x}{f(x)}$ is called the \emph{Fr\'{e}chet derivative of $f$ in $x$.}
\end{definition}
\begin{remark}
When it is clear from context that $y=f(x)$, we will simply write $\D{x}{y}$ for $\D{x}{f(x)}$.
\end{remark}
\begin{definition}\label{def:gradient}
Let $\X$ be a real Hilbert space and let $U \subset \X$ be an open subset. Let further $f: \X \rightarrow \mathbb{R}$ be Fr\'{e}chet differentiable in $x\in U$. We call the vector $v\in U$ for which $\D{x}{f(x)} \cdot h = \langle v,h \rangle_\X$ for all $h\in U$ the \emph{gradient of $f$ with respect to $x$} and write $\nabla_xf(x):=v$.
\end{definition}
\begin{remark}
The existence and uniqueness is guaranteed by Riesz' representation theorem.
\end{remark}

\begin{theorem}[Generalized product rule]\label{thm:product_rule}
Let $\X$, $\Y_1$, $\Y_2$ and $\mathcal{Z}$ be Hilbert spaces. Let $f: \X \rightarrow \Y_1$ and $g: \X \rightarrow \Y_2$ be Fr\'{e}chet differentiable in $x \in U \subset \X$ ($U$ open in $\X$) and let $B: \Y_1 \times \Y_2 \rightarrow \mathcal{Z}$ be a continuous bilinear operator. Then $B(f(\ph),g(\ph))$ is Fr\'{e}chet differentiable in $x$ and 
\begin{equation}
    \begin{aligned}
&\D{x}{B\left( f(x), g(x) \right)} \cdot h \\
=&B\left( \D{x}{f(x)} \cdot h, g(x) \right) + B\left( f(x), \D{x}{g(x)} \cdot h\right)    
    \end{aligned}
\end{equation}
\end{theorem}

\begin{theorem}[Chain rule]\label{thm:chain_rule}
Let $\X, \Y$ and $\mathcal{Z}$ be Hilbert spaces. Let $f: \X \rightarrow Y$ and $g: U \rightarrow \mathcal{Z}$ for some open set $U \subset \Y$. Let $f$ be Fr\'{e}chet differentiable in $x$ and $y:=f(x) \in U$. Let further $g$ be differentiable in $y$ and $z:=g(y)$. Then $g \circ f$ is well-defined in $x$ and Fr\'{e}chet differentiable in $x$ with
$$\D{x}{g(y)} = \D{y}{g(y)}\D{x}{y}.$$
\end{theorem}
Applying the chain rule to Definition \ref{def:gradient} of the gradient yields the following theorem.
\begin{theorem}[Gradient chain rule]\label{thm:gradient_chain_rule}
Let the assumptions from Theorem \ref{thm:chain_rule} hold, where additionally $\mathcal{Z}=\mathbb{R}$. Then
$$\nabla_x z = \left( \D{x}{y} \right)^\ast \cdot \nabla_y z.$$
\end{theorem}

In $\mathbb{R}^d$, the Fr\'{e}chet derivative can be represented as a matrix, the \emph{Jacobian}, given continuity of the partial derivatives. The gradient definition in \ref{def:gradient} leads to the familiar column vector consisting of partial derivatives.

\section{Neural Network Model}\label{sec:neural_net_model}
\subsection{Forward Pass}
In the following, we will consider an $L$-layer network
\begin{equation}
\begin{rcases}
 z_j &= K_j(\thetaj,x_{j-1}) + b_j \hspace{2mm} \\
x_j &= g_j(\zj)
\end{rcases} \text{ for } j=1,\dots,L,
\end{equation}
where $x_0$ is the input to the network. Here, $K_j : \P_j \times \X_{j-1} \rightarrow \X_j$ is a continuous bilinear operator between Hilbert spaces, $\thetaj \in \P_j$ and $b_j \in \X_j$ are the $j$-th layer's parameters, $\zj\in \X_j$ and $\xj\in\X_j$ its respective pre-activation and activation, $g_j : \X_j \rightarrow \X_j$ the activation function. Here we concentrate on networks with $\X_L=\mathbb{R}^C$ (with standard inner product), e.g. classifiers with $g_L=\text{softmax}$, but the results naturally extend to other types of neural networks. 

\subsection{Dealing with Nonlinearities}
We will further assume $g_j$ (for $j<L$) to be a nonlinearity that is applied 'coordinate-wise', like ReLU or tanh. Assuming some coordinate representation would of course defeat the purpose of finding a coordinate-free (double) backpropagation scheme. This is why we have to find a more general characterization of these types of functions that still retains their simplicity.
\begin{definition}[Coordinate-wise Fr\'{e}chet differentiable] \label{def:coordinate_frechet}
Let $\X$ be a real Hilbert space and $U \subset \X$ be an open subset. If there exists a symmetric, bilinear operator $M: \mathcal{X} \times \X \rightarrow \X$ such that $g: U \rightarrow \X$ is Fr\'{e}chet differentiable in $x\in U$ with $$\D{x}{g(x)} \cdot v = M(g'(x), v)$$ for some function $g': U \rightarrow \X$ for all $v\in\X$, we call $g$ \emph{coordinate-wise Fr\'{e}chet differentiable} in $x$. If $g'$ is itself coordinate-wise Fr\'{e}chet differentiable in $x$, we call $g$ coordinate-wise twice Fr\'{e}chet differentiable in $x$ with second derivative $g''$.
\end{definition}
The motivation behind this technical definition is the fact that for the coordinate-wise application of functions like $g=\text{tanh}$, the Jacobian is a diagonal matrix, such that $$\D{x}{\text{tanh}(x)} \cdot v = \text{tanh}'(x) \odot v$$ with $M: (x,y) \mapsto x \odot y$ denoting the coordinate-wise multiplication. When appropriate, we will use the abbreviations
\begin{equation}
    G'(x):= \D{x}{g(x)} \text{ and } G''(x):= \D{x}{g'(x)},
\end{equation}
which allows us to easily switch between viewing these derivatives as either linear or bilinear maps. The latter will later be essential in order to be able to apply the generalized product rule \ref{thm:product_rule}. The following lemma and corollary (proof in Appendix \ref{sec:appdx_bilinear}) show that these functions are self-adjoint.
\begin{restatable}{lemma}{symmetricOpsSelfAdjoint}\label{lem:symmetric_ops_self_adjoint}
Let $M: \X \times \X \rightarrow \X$ be a symmetric, bilinear operator. Then $M=M^\wup=M^\transp$.
\end{restatable}

\begin{restatable}{corollary}{multiplicationOpSelfAdjoint}\label{cor:multiplication_op_self_adjoint}
Let $A:=M(a,\ph)$, where $M: \X \rightarrow \X$ is a symmetric, bilinear operator and $a\in\X$. Then $A$ is self-adjoint.
\end{restatable}

We point out that according to Lemma \ref{cor:multiplication_op_self_adjoint}, $G'(x)$ and $G''(x)$ are self-adjoint operators if $g$ is coordinate-wise Fr\'{e}chet (twice) differentiable in $x$.\newline
The restriction to coordinate-wise nonlinearities (except for the final layer) allows for a great simplification of the utilized theory, while representing the vast majority of real-world neural networks. In particular, the product rule can be readily applied. If on the other hand, one were to use general Fr\'{e}chet differentiable activation functions, the used higher-order derivatives would need to be calculated on spaces of Fr\'{e}chet derivatives\footnote{In $\mathbb{R}^d$, these can be represented as tensors of order up to 4.}, which demands a much more involved derivation of the backpropagation rules.

\section{Deriving Double Backpropagation Rules}\label{sec:deriving_double}
Double backpropagation comes into play, whenever the loss function to be minimized contains a derivative of a function with respect to $x_0$. As we optimize our loss using first-order methods, our ultimate goal is to determine the gradients $\nabla_{\theta_j} \R \in \P_j$ and $\nabla_{b_j} \R \in \X_j$, where $\R$ denotes an expression that depends on a derivative with respect to $x_0$ (usually a regularization or penalty term).
\subsection{Penalty Terms} \label{sec:penalty_terms}
We consider penalty terms $\R$ (or sums thereof) that can be written in the form 
\begin{equation}\begin{aligned} \label{eq:penalty_general_form}
    \R:=p\left( \left(\D{x_0}{x_L} \right)^\ast \cdot v \right),
\end{aligned}\end{equation}
where $p: \X_0 \rightarrow \mathbb{R}$ is differentiable almost everywhere and not locally constant and $v$ may or may not depend on $x_L,\dots,x_0$. The exact form of the penalty is thus determined by $p$ and $v$. In the following, we will offer some examples.
\subsubsection{Classical double backpropagation} \label{sec:classical_dbp}
In classical double backpropagation, we apply a penalty $\|\nabla_{x_0} \L \|_{\X_0}^2$, where $\L:=\ell (x_L,y)$ is the network's loss (with loss function $\ell$). Here $y \in \mathbb{R}^C = \X_L$, e.g. a one-hot encoded label vector. By applying the gradient chain rule (Theorem \ref{thm:gradient_chain_rule}), this yields
\begin{equation}
    \begin{aligned}
        &\hspace{.9mm}\|\nabla_{x_0} \L \|_{\X_0}^2 \\
        =& \left\| \left(\D{x_0}{x_L} \right)^\ast\cdot \nabla_{x_L} \ell (x_L, y)  \right\|_{\X_0}^2,
    \end{aligned}
\end{equation}
so that $p: u \mapsto \|u\|_{\X_0}^2$ and $v=\nabla_{x_L} \L$. In the special case of the squared euclidean error 
\begin{equation}\label{eq:NLL}
    \ell(x_L,y)=\|x_L - y\|_2^2,    
\end{equation}
this results in $v=2(x_L-y)$. When using the negative log-likelihood $$\ell(x_L,y) = - \sum_{i=1}^C y^i \cdot \log \left( x^i_L\right),$$ we get $v=-y \oslash x_L$ (with $\oslash$ denoting the component-wise (\emph{Hadamard}) division). These constitute cases where $v$ depends on $x_L$ and thus on all $x_j$ with $j\leq L$.
\subsubsection{Penalties on gradients of output nodes}
Another general type of penalty is on derivatives of output nodes with respect to the input. For example, $\|\nabla_{x_0} x^i_L\|_2^2$, a squared euclidean norm penalty on the gradient of the $i$-th output node with respect to the input, can be represented via $v= e^{\langle i\rangle}$ in $\eqref{eq:penalty_general_form}$, where $e^{\langle i\rangle}$ denotes the $i$-th standard unit vector.\newline
We can immediately obtain formulas for the (squared) Frobenius norm of the Jacobian $J_f \simeq \D{x_0}{x_L}$ by realizing that $$\|J_f\|_F^2= \sum\limits_{i=1}^C \|\nabla_{x_0} x_L^i\|_2^2,$$ which entails $C$ penalties of the form \eqref{eq:penalty_general_form}. This however naturally increases the time complexity of the double backpropagation by a factor of $C$. We will later present an algorithm, with which the runtime may be reduced by up to a third, depending on the used activation functions.\newline
If the penalties are applied on the logits ($z_L$) instead of the $\text{softmax}$-outputs ($x_L$), one can simply model this via $g_L=\text{id},$ the identity function.

\subsubsection{Operator norm penalties}
In section \ref{sec:examples}, we mentioned how randomized penalties can be employed in the calculation of the spectral norm of the Jacobian (more generally: operator norms of the Fr\'{e}chet derivative). The operator norm can be written as
\begin{equation}\begin{aligned}\label{eq:operator_norm}
    \left\| \D{x_0}{x_L} \right\|_{\X_0 \rightarrow \X_L} :=& \sup\limits_{\|u\|_{\X_0}=1} \left\| \D{x_0}{x_L} \cdot u \right\|_2 \\
     =& \sup\limits_{\|v\|_2=1} \left\| \left( \D{x_0}{x_L} \right)^\ast \cdot v \right\|_{\X_0},
\end{aligned}\end{equation}
where we used the fact that the operator norms of primal and dual bounded operators in Hilbert spaces coincide \citep{rudin1991functional}. By sampling $\tilde{v}$ from a normal distribution and setting $v=\tilde{v}/\|\tilde{v}\|_2$, one samples $v$ almost surely uniformly from the unit sphere $\lbrace v : \|v\|_2=1 \rbrace$ \citep{muller1959note}. With $p=\|\ph\|_{\X_0}$, we thus obtain a lower bound of the operator norm, which yields a penalty term of the form \eqref{eq:penalty_general_form}. This is equivalent to one power iteration. Better estimates of the optimal $v$ in \eqref{eq:operator_norm} are obtained by performing multiple power iterations.

\subsection{Backward Pass: Calculating the Penalty Terms}
In order to calculate $\R$ in the first place, we define 
\begin{equation}\begin{aligned}
    \xi_j :&= \left(\D{x_j}{x_L} \right)^\ast \cdot v \\
    \zeta_j :&= \left(\D{z_j}{x_L} \right)^\ast \cdot v, \\
\end{aligned}\end{equation}
which allows us to write $\xi_L=v$ and $\R = p(\xi_0).$ Given $\xi_j$, we can calculate $\zeta_j$ via
\begin{equation}\begin{aligned}\label{eq:zetaj_formula}
    \zeta_{j} &= \left(\D{z_{j}}{x_L} \right)^\ast \cdot v \\
    &= \left(\D{z_{j}}{x_j} \right)^\ast \left(\D{x_j}{x_L} \right)^\ast \cdot v \\
    &= \left(G'(z_j) \right)^\ast \cdot \xi_j \\
    &= G'(z_j)  \cdot \xi_j ,
\end{aligned}\end{equation}
using the chain rule and the self-adjointness of $G'(z_j)$. Given $\zeta_j$, we can further calculate
\begin{equation}\begin{aligned} \label{eq:xijminusone_formula}
    \xi_{j-1} &= \left(\D{x_{j-1}}{x_L} \right)^\ast \cdot v \\
    &= \left(\D{x_{j-1}}{z_j} \right)^\ast \left(\D{z_j}{x_L} \right)^\ast \cdot v \\
    &= \left(K_j(\theta_j, \ph) \right)^\ast \cdot \zeta_j \\
    &= K^T_j(\theta_j, \ph) \cdot \zeta_j \\
    &= K^T_j(\theta_j, \zeta_j) \cdot , \\
\end{aligned}\end{equation}
where we applied the chain rule and used the fact that the adjoint of $K_j(\theta_j,u)$ in $u$ is the transposed operator of $K_j$. In summary, the penalty is calculated via the recursion given in Algorithm \ref{alg:penalty_calc}.

\begin{algorithm}[tb]
   \caption{Calculation of the penalty term}\label{alg:penalty_calc}
\begin{algorithmic}
   \STATE {\bfseries Initialize} $\xi_L = v$
   \FOR{$j\gets L$ {\bfseries to}  $1$}
   \STATE $\zetaj = G_j'(z_j)\cdot \xi_j$
   \STATE $\xi_{j-1} = K_j^\transp \left(\thetaj, \zetaj \right)$
   \ENDFOR
   \STATE {\bfseries Output:} $\R = \hspace{1pt} p( \xi_0 )$
\end{algorithmic}
\end{algorithm}

Here, the difficulty in calculating $\nabla_{\theta_j}\R$ and $\nabla_{b_j}\R$ for all $j$ becomes visible: While $\xi_{j-1}$ depends directly on $\theta_j$, it also depends on $\zetaj$, which itself directly depends on $z_j$, which in turn depends on $\theta_j$. Furthermore, $\zeta_j$ depends on $\xi_j$, which implicitly depends on $\theta_j$ as well, since it is a result of the backward pass. In other words, due to the edges from the upper half to the lower half of the graph, $\zetaj$ depends on every variable except for $\xi_0,\dots,\xi_{j-1},\zeta_1,\dots,\zeta_{j-1},x_L,\R$ and $\L$. For the calculation of the weight-gradients, one hence has to untangle these complicated functional relationships. The complete interdependence of all involved variables is displayed in the dependency graph in Figure \ref{fig:dependency_graph}.
\begin{figure*}\centering
\def\u{1.5cm}
\begin{tikzpicture}
\usetikzlibrary{calc}
\tikzstyle{every node}=[font=\small]

\tikzset{
    influences/.style={->, >=latex, shorten >=-2pt},
    maybeinfluences/.style={dashed,->, >=latex, shorten >=-2pt},
    beginning/.style={-, >=latex, shorten >=-1pt},
    middle/.style={dotted, >=latex, shorten >=-1pt},
    end/.style={->, >=latex, shorten >=1pt},
}  

\node (R) at (-3.5*\u,-1*\u) {$\mathcal{R}$};
\node (L) at (6.5*\u,0) {$\mathcal{L}$};

\node (x0) at (-2.5*\u,0) {$x_0$};
\node (xm) at (-1*\u,0) {$x_{j-1}$};
\node (zj) at (0*\u,0) {$z_j$};
\node (xj) at (1*\u,0) {$x_j$};
\node (zjj) at (2*\u,0) {$z_{j+1}$};
\node (xjj) at (3*\u,0) {$x_{j+1}$};
\node (zL) at (4.5*\u,0) {$z_L$};
\node (xL) at (5.5*\u,0) {$x_L$};

\node (bj) at (-0.5*\u,.5*\u) {$b_j$};
\node (bjj) at (1.5*\u,.5*\u) {$b_{j+1}$};
\node (bL) at (4*\u,.5*\u) {$b_{L}$};

\node (r0) at (-2.5*\u,-1*\u) {$\xi_0$};
\node (rm) at (-1*\u,-1*\u) {$\xi_{j-1}$};
\node (dj) at (0*\u,-1*\u) {$\zeta_j$};
\node (rj) at (1*\u,-1*\u) {$\xi_j$};
\node (djj) at (2*\u,-1*\u) {$\zeta_{j+1}$};
\node (rjj) at (3*\u,-1*\u) {$\xi_{j+1}$};
\node (dL) at (4.5*\u,-1*\u) {$\zeta_L$};
\node (xiL) at (5.5*\u,-1*\u) {$\xi_L$};
\node (v) at (6.5*\u,-1*\u) {$v$};

\node (tj) at (-.5*\u,-.5*\u) {$\theta_j$};
\node (tjj) at (1.5*\u,-.5*\u) {$\theta_{j+1}$};

\draw[influences] (xm) -- (zj);	
\draw[influences] (zj) -- (xj);	
\draw[influences] (xj) -- (zjj);
\draw[influences] (zjj) -- (xjj);

\draw[influences] (zj) -- (dj);	
\draw[influences] (zjj) -- (djj);
\draw[influences] (zL) -- (dL);	
\draw[influences] (zL) -- (xL);	
\draw[influences] (xiL) -- (dL);	

\draw[influences] (rjj) -- (djj);
\draw[influences] (djj) -- (rj);
\draw[influences] (rj) -- (dj);
\draw[influences] (dj) -- (rm);

\draw[influences] (tj) -- (zj);
\draw[influences] (tj) -- (rm);
\draw[influences] (tjj) -- (zjj);
\draw[influences] (tjj) -- (rj);

\draw[influences] (bj) -- (zj);
\draw[influences] (bjj) -- (zjj);
\draw[influences] (bL) -- (zL);

\draw[influences] (r0) -- (R);
\draw[influences] (xL) -- (L);
\draw[influences] (v) -- (xiL);

\draw[maybeinfluences] (xL) -- (xiL);
\draw[maybeinfluences] (xL) -- (dL);

\draw[beginning] (xjj) -- ($(xjj) + (.5*\u,0)$);
\draw[middle] ($(xjj) + (.5*\u,0)$) -- ($(xjj) + (1.0*\u,0)$);
\draw[end] ($(xjj) + (1.0*\u,0)$) -- ($(xjj) + (1.4*\u,0)$);

\draw[beginning] (dL) -- ($(dL) - (.5*\u,0)$);
\draw[middle] ($(dL) - (.5*\u,0)$) -- ($(dL) - (1.0*\u,0)$);
\draw[end] ($(dL) - (1.0*\u,0)$) -- ($(dL) - (1.3*\u,0)$);

\draw[beginning] (x0) -- ($(x0) + (.5*\u,0)$);
\draw[middle] ($(x0) + (.5*\u,0)$) -- ($(x0) + (1.0*\u,0)$);
\draw[end] ($(x0) + (1.0*\u,0)$) -- ($(x0) + (1.27*\u,0)$);

\draw[beginning] (rm) -- ($(rm) - (.5*\u,0)$);
\draw[middle] ($(rm) - (.5*\u,0)$) -- ($(rm) - (1.0*\u,0)$);
\draw[end] ($(rm) - (1.0*\u,0)$) -- ($(rm) - (1.4*\u,0)$);

\draw[middle] ($(zL)-(.3*\u,.3*\u)$) -- ($(zL)-(.2*\u,.2*\u)$);
\draw[end] ($(zL)-(.2*\u,.2*\u)$) -- ($(zL)-(.05*\u,.05*\u)$);

\draw[middle] ($(rjj)+(.3*\u,.3*\u)$) -- ($(rjj)+(.2*\u,.2*\u)$);
\draw[end] ($(rjj)+(.2*\u,.2*\u)$) -- ($(rjj)+(.05*\u,.05*\u)$);
\end{tikzpicture}
\caption{Dependency graph of the quantities in the derivative-regularized network according to sections \ref{sec:neural_net_model} and \ref{sec:deriving_double}. An edge from node $A$ to node $B$ signifies $B$ being a function of $A$. This implies that $B$ is also a function of every node that $A$ is a function of etc. The input nodes are $x_0$ and $v$, while the output nodes are $L$ and $R$. Dashed lines symbolize \emph{possible} dependencies, which depends on the exact loss function used. See Appendix \ref{sec:appdx_initial_etaL} for information about this.}\label{fig:dependency_graph}
\end{figure*}
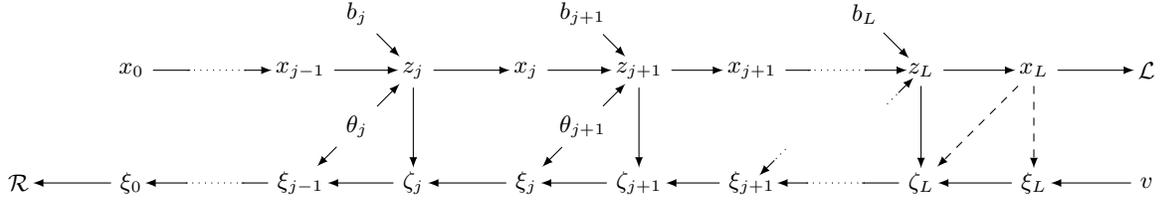

\subsection{Backward-Backward Pass}
Networks that require double backpropagation can be viewed as extended neural networks, where the forward pass (FP) and the backward pass (BP) are concatenated to form the forward pass of a neural network with twice the depth. Through this lens, double backpropagation is nothing but backpropagation through the extended network, where the gradients first pass through the BP of the original network, then the FP or the original network (which was already recognized in \citep{drucker1992improving}). We therefore call the procedures with which we calculate these gradients the \emph{backward-backward pass} and the \emph{forward-backward pass}.

Much like in standard backpropagation, our goal is to calculate $\nabla_{\theta_j}\R$ and $\nabla_{b_j}\R$, while keeping the dependency graph (Fig. \ref{fig:dependency_graph}) in mind. Due to
\begin{equation}\begin{aligned}\label{eq:gradient_formulas}
\nabla_{\theta_j} \R &= \left( \D{\thetaj}{\xi_{j-1}} \right)^\ast \cdot \nabla_{\xi_{j-1}} \R \\
\nabla_{b_j} \R &= \left( \D{b_j}{\zeta_{j}} \right)^\ast \cdot \nabla_{\zeta_{j}} \R
\end{aligned}\end{equation}
we are interested in
\begin{equation}
    \begin{aligned}
        \gvar_{j} :=& \nabla_{\xi_j} \R\\
        \hh_j :=& \nabla_{\zeta_j} \R,
    \end{aligned}
\end{equation}
for which we need backpropagation rules. From equations \eqref{eq:zetaj_formula} and \eqref{eq:xijminusone_formula} we can infer that 
\begin{equation}
\begin{aligned}
\D{\zeta_j}{\xi_{j-1}}&=\left(\D{x_{j-1}}{z_j}\right)^\ast \\
\D{\xi_j}{\zeta_{j}}&=\left(\D{z_{j}}{x_j}\right)^\ast,
\end{aligned}
\end{equation}
so that
\begin{equation}\begin{aligned}
\hh_j =& \nabla_{\zeta_j} \R \\
=& \left( \D{\zeta_j}{\xi_{j-1}} \right)^\ast \cdot \nabla_{\xi_{j-1}} \R \\
=& \hphantom{.} K^T_j(\theta_j,\ph)^\ast \cdot \gvar_{j-1} \\
=&\hphantom{.} K_j(\theta_j, \gvar_{j-1})
\end{aligned}\end{equation}
 and
\begin{equation}\begin{aligned}
\gvar_{j} =& \nabla_{\xi_{j}} \R \\
=& \left( \D{\xi_{j}}{\zeta_j} \right)^\ast \cdot \nabla_{\zeta_j} \R \\
=& \hphantom{.} {G'(z_j)}^\ast \cdot \hh_j \\
=& \hphantom{.} G'(z_j) \cdot \hh_j,
\end{aligned}\end{equation}
which results in the iteration scheme summarized in Algorithm \ref{alg:back_back}.\\

\begin{algorithm}[tb]
   \caption{Calculation of the backwards-backwards terms}\label{alg:back_back}
\begin{algorithmic}
   \STATE {\bfseries Initialize} $\gvar_0= \nabla_{\xi_0} p ( \xi_0)$
   \FOR{$j\gets 1$ {\bfseries to}  $L$}
   \STATE $\hh_j = K_j(\theta_j,\gvar_{j-1})$
   \STATE $\gvar_j = G'(z_j) \cdot \hh_j$
   \ENDFOR
\end{algorithmic}
\end{algorithm}

At this point, we still cannot evaluate equations \eqref{eq:gradient_formulas}. As visible in Figure \ref{fig:dependency_graph}, the linear operators 
$$\D{\thetaj}{\xi_{j-1}} \text{ and }  \D{b_j}{\zeta_j}$$
make us consider the functional relationship between the upper and lower half of the dependency graph. This happens through the forward-backward pass.

\subsection{Fordward-Backward Pass}
We continue to try to evaluate equations \eqref{eq:gradient_formulas}. First off, we note that calculating the bias-gradients 
\begin{equation}\begin{aligned}\label{eq:b_gradient_formula}
\nabla_{b_j} \R &= \left( \D{b_j}{z_{j}} \right)^\ast \left( \D{z_j}{\zeta_{j}} \right)^\ast \cdot \nabla_{\zeta_{j}} \R  \\
 &= id^\ast \cdot \left(\D{z_j}{\zeta_j}\right)^\ast \cdot \hh_j\\
 &=\left(\D{z_j}{\zeta_j}\right)^\ast \cdot \hh_j
\end{aligned}\end{equation}
requires evaluating $$\eta_j:=\left(\D{z_j}{\zeta_j}\right)^\ast \cdot \hh_j.$$ Note that due to $\hh_j=\nabla_{\zeta_{j}} \R$, one has $\eta_j=\nabla_{z_j} \R$. Right now, we do not have a way of evaluating $\eta_j$ yet, but we will derive an expression for it later. \\
Similarly to the bias-gradients, we express the gradients of the linear weights as 
\begin{equation}\begin{aligned}\label{eq:theta_gradient_formula}
\nabla_{\theta_j} \R &= \left( \D{\thetaj}{\xi_{j-1}} \right)^\ast \cdot \nabla_{\xi_{j-1}} \R  \\
 &= \left( \D{\thetaj}{\xi_{j-1}} \right)^\ast \cdot \gvar_{j-1}.  \\
\end{aligned}\end{equation}

Since $\xi_{j-1}=K_j^T(\theta_j,\zeta_j)$ and because $K_j^T$ is a continuous bilinear operator, we can harness the generalized product rule (Theorem \ref{thm:product_rule}) for equation \eqref{eq:theta_gradient_formula}:
\begin{equation}\begin{aligned}
\D{\thetaj}{\xi_{j-1}}^\ast &= \left( K^\transp_j\left(\D{\thetaj}{\thetaj} \ph, \zetaj\right) + K^\transp_j\left(\thetaj, \D{\thetaj}{\zetaj} \ph\right) \right)^\ast\\
 &= K^\transp_j(\ph, \zetaj)^\ast + K^\transp_j\left(\thetaj, \D{\thetaj}{\zetaj} \ph \right)^\ast \\
 &= K^\wup_j(\ph, \zetaj) + \left(\D{\thetaj}{\zetaj} \right)^\ast \cdot K_j\left(\thetaj,  \ph \right), \\
\end{aligned}\end{equation}
where we used the anti-distributivity of adjoint operators. According to Figure \ref{fig:dependency_graph}, $\zeta_j$ only depends on $\theta_j$ through $z_j$. We hence can apply the chain rule
$$\D{\thetaj}{\zetaj}=\D{z_j}{\zeta_j}\D{\thetaj}{z_j}=\D{z_j}{\zeta_j} \cdot K_j(\ph,x_{j-1})$$

Plugging this into equation \eqref{eq:theta_gradient_formula} yields 
\begin{equation}\begin{aligned}
    \nabla_{\theta_j} \R =& K^\wup_j(\gvar_{j-1}, \zetaj)^\ast + \left(\D{\thetaj}{\zetaj} \right)^\ast \cdot K_j\left(\thetaj,  \gvar_{j-1} \right) \\
   =& K^\wup_j(\gvar_{j-1}, \zetaj) +  \left(\D{\thetaj}{z_j} \right)^\ast \left(\D{z_j}{\zeta_j} \right)^\ast \cdot \hh_j\\
   =& K^\wup_j(\gvar_{j-1}, \zetaj) +  K^\wup_j\left(\eta_j,x_{j-1}\right) ,
\end{aligned}\end{equation}
which means that for both $\nabla_{b_j}\R$ and $\nabla_{\theta_j}\R$, 
we need a way of evaluating $\eta_j=\nabla_{z_j}\R$.\\

Here, we may finally make use of the fact that we assumed $g_j$ to be coordinate-wise Fr\'{e}chet differentiable (Def. \ref{def:coordinate_frechet}) for all $j<L$ (with some symmetric, bilinear operator $M_j$). This allows us to write
\begin{equation}\begin{aligned}\label{eq:D-zeta-theta}
\D{z_j}{\zeta_j} =& \D{z_j}{G'_j(\zj)\cdot\xi_j}= \D{z_j}{M_j(g'_j(z_j),\xi_j)} \\
=& M_j\left( \D{z_j}{g'_j(z_j)}  \ph,\xi_j \right) + M_j\left(g'_j(\zj), \D{z_j}{\xi_j} \ph \right)
\end{aligned}\end{equation}

Applying the adjoint
\begin{equation}\begin{aligned}\label{eq:D-zeta-theta-adjoint}
\D{z_j}{\zeta_j}^\ast 
=& G_j''(z_j) \cdot  M_j\left( \ph, \xi_j \right) + \left(\D{z_j}{\xi_j}\right)^\ast  \cdot G'(z_j),
\end{aligned}\end{equation}
with 
$$\left(\D{z_j}{\xi_j}\right)^\ast=\left(\D{z_j}{x_j}\right)^\ast\left(\D{x_j}{\xi_j}\right)^\ast$$
to $\hh_j$ as in equations \eqref{eq:theta_gradient_formula} and \eqref{eq:b_gradient_formula} yields 
\begin{equation}\begin{aligned} \label{eq:eta_formula}
    \eta_j =&M_j\left( G_j''(z_j) \cdot \hh_j,\xi_j \right) + G_j'(z_j) \cdot \nabla_{x_j} \R\\
    =&M_j\left( G_j''(z_j) \cdot \hh_j,\xi_j \right) + G_j'(z_j) \cdot \gamma_j
\end{aligned}\end{equation} 
with
\begin{equation}
    \begin{aligned}
        \gamma_j := \nabla_{x_j} \R .
    \end{aligned}
\end{equation}

While we do not have an expression for $\gamma_j$ yet, we can recursively calculate it via the formula 
\begin{equation}\begin{aligned}\label{eq:gamma_formula}
    \gamma_{j-1}=& \nabla_{x_{j-1}} \R \\
    =& \left( \D{x_{j-1}}{z_j} \right)^\ast \nabla_{z_{j}} \R\\
    =& K_j^T(\theta_j, \eta_j)
\end{aligned}\end{equation}
for which the initial value $\eta_L$ is required, which depends on the exact penalty term (see Appendix \ref{sec:appdx_initial_etaL}). Equipped with this, the weight gradients \begin{equation}
	\begin{aligned}
		\nabla_\thetaj\R =& K^\wup_j(\gvar_{j-1}, \zetaj) + K^\wup_j\left(\eta_j,x_{j-1}\right)\\
		\nabla_{b_j}\R =& \eta_j
	\end{aligned}
\end{equation}
can be calculated. This procedure is summarized in Algorithm \ref{alg:weight_gradients}.\\

\begin{algorithm}[tb]
  \caption{Calculation of the forward-backward pass and the weight-gradients}\label{alg:weight_gradients}
\begin{algorithmic}
   \FOR{$j\gets L$ {\bfseries to}  $1$}
   \IF{$j=L$}
   \STATE {\bfseries Initialize} $\eta_L$ according to Appendix \ref{sec:appdx_initial_etaL}
   \ELSE
   \STATE $\eta_j=M_j\left( G_j''(z_j) \cdot \hh_j,\xi_j \right) + G_j'(z_j) \cdot \gamma_j$
   \ENDIF
   \STATE $\nabla_{\theta_j}\R = K^\wup_j(\gvar_{j-1}, \zetaj) + K^\wup_j\left(\eta_j,x_{j-1}\right)$
   \STATE $\nabla_{b_j}\R = \eta_j$
   \IF{$j>1$}
   \STATE $\gamma_{j-1}=K_j^T(\theta_j, \eta_{j})$
   \ENDIF
   \ENDFOR
\end{algorithmic}
\end{algorithm}

\subsection{Standard Backpropagation}\label{sec:standard_backpropagation}
We can easily recover the standard backpropagation rules (without any derivative-based penalty terms) for the loss $\ell(x_L,y)$ from the above setup by setting $v=\nabla_{x_L} \ell (x_L, y)$. Then 
\begin{equation}\begin{aligned}
\nabla_{\theta_j} \ell(x_L,y) &= \left( \D{\theta_j}{z_j}  \right)^\ast \cdot \nabla_{z_j} \ell(x_L,y) \\
&= \left( \D{\theta_j}{K_j(\theta_j,x_{j-1}) + b_j}  \right)^\ast \cdot \zeta_j \\
&= \left( K_j(\ph,x_{j-1}) \right)^\ast \cdot \zeta_j \\
&= K^\wup_j (\zeta_j,x_{j-1})
\end{aligned}\end{equation}
and
\begin{equation}\begin{aligned}
\nabla_{b_j} \ell(x_L,y) &= \left( \D{b_j}{z_j}  \right)^\ast \cdot \nabla_{z_j} \ell(x_L,y) \\
&= \left( \D{b_j}{K_j(\theta_j,x_{j-1}) + b_j}  \right)^\ast \cdot \zeta_j \\
&= id^\ast \cdot \zeta_j \\
&= \zeta_j,
\end{aligned}\end{equation}
which provide the well-known weight-gradients, that are needed for each iteration of a first-order optimization scheme of the network's loss, in the general framework of continuous bilinear operators.\\
This also demonstrates that one is able to 'reuse' the values $\zeta_i$ and $\xi_i$ for standard backpropagation \emph{and} for classical double backpropagation, unlike for all other penalty terms.

\section{Runtimes}
In the last section, the double backpropagation rules were derived. For most networks (in particular in convolutional neural networks), the most time-consuming portion of the network lies in the calculation of the forward and transposed operators $K_j$ and $K^T_j$. Here, we will consider the runtimes of different penalty terms and offer optimized implementations of some.
\subsection{The General Case}
In the general case, the forward, backward and backward-backward pass each require $L$ evaluations of the (transposed) operators. The forward-backward pass however does not require to evaluate $\gamma_0=K^T_1(\theta_1,\eta_1)$ , which is why in this case only $L-1$ transposed operations need to be performed. This results in a time complexity of $4L-1$ operations for the full double backpropagation.\newline
If the full loss term is $\L+\lambda p(\xi_0)$, but $\xi_0$ is \emph{not} $\nabla_{x_0}\L$ (as in classical double backpropagation), one needs to perform another $L-1$ operations, because $\nabla_{\theta_j}\L$ and $\nabla_{b_j}\L$ are needed, whereas some values can be reused in classical double backpropagation (as detailed in section \ref{sec:standard_backpropagation}). In summary, these cases require $5L-2$ linear operations, compared to the $2L-1$ operations of a network without a penalty term of the type \eqref{eq:penalty_general_form}.\\

\subsection{Locally Linear Activation Functions}
If the activations $x_j=g_j(z_j)$ are not only coordinate-wise twice Fr\'{e}chet differentiable in $z_j$, but also locally linear (as with the popular (leaky) rectified linear unit \emph{ReLU}) in $z_j$, the double backpropagation takes a simpler form: \newline
As $G''(z_j)$ is the null operator (\emph{almost everywhere}, for every $z_j$ for which $g_j(z_j)$ is twice Fr\'{e}chet differentiable) and because $M_j$ is linear in both arguments, equation \eqref{eq:eta_formula} reduces to 
\begin{equation}
    \eta_j = G'_j(z_j) \cdot \gamma_j.
\end{equation}
In general, this however does not reduce the amount of linear operations $K_j$ and $K_j^T$.

\subsection{Linear Output Nodes and Locally Linear Activation Functions}
If the penalty terms are applied on the derivatives of linear output nodes (i.e. $g_L=\text{id}$, for example when $\R=\|\nabla_{x_0}z_L^i\|_{\X_0}^2$), then $\eta_L=0$, as demonstrated in Appendix \ref{sec:appdx_initial_etaL}. While this only reduces the amount of linear operations by 1 (through $\gamma_{L-1}=K_L^T(\theta_j,0)=0$), the effect cascades when also a locally linear activation function is used. This is because in that case, $\eta_j=\gamma_j=0$ for all $j$, according to equations \eqref{eq:eta_formula} and \eqref{eq:gamma_formula}. As a result, the weight gradients reduce to $\nabla_{b_j}\R=0$ for all $j$ and $\nabla_{\theta_j}\R=K^\wup_j(\gvar_{j-1},\zeta_j)$, which means that one does not need to perform the forward-backward pass at all. All in all, the reduced number of linear operations for the penalty term is then $3L$, and $4L-1$ for the full loss term $\L+\lambda \R$ (the same as for the classical double backpropagation loss $L+\lambda p(\nabla_{x_0}\L)$).

\subsection{Jacobian Penalties}
If the penalty term is $$\R=\sum\limits_{i=1}^C \R^\ith :\equiv \sum\limits_{i=1}^C \|\nabla_{x_0} x_L^i\|_2^2$$ (equivalent to the squared Frobenius norm of the Jacobian), the double backpropagation scheme needs to be performed $C$ times. Note that the forward pass has to be performed only once. The total number of linear operations is thus $L+C(3L-1)$, plus another $L-1$ if one needs $\nabla_{\theta_j}\L$ and $\nabla_{b_j}\L$ (which results in $2L-1+C(3L-1)$ linear operations). \\

We now present an optimized double backpropagation algorithm for this scenario, that applies if only locally linear activation functions like ReLU (up to the final softmax layer) are employed, which allows one to abuse a certain linearity. This algorithm reduces the number of performed linear operations by almost a third, while keeping the required memory roughly the same as with a single conventional double backpropagation. We will now index variables such as $\eta_j$ that relate to a certain $\R^\ith$ via $\eta_j^\ith$.
We note the following:
\begin{enumerate}
    \item We can write $\nabla_\thetaj\R = \hat{\theta}^1_j + \hat{\theta}^2_j,$ where $\hat{\theta}^1_j=\sum\limits_{i=1}^C K^\wup_j(\gvar^\ith_{j-1}, \zetaj^\ith)$ and $\hat{\theta}^2_j= \sum\limits_{i=1}^C K^\wup_j\left(\eta^\ith_j,x_{j-1}\right)$.
    \item When looping over $i$, $\hat{\theta}^1_j$ can be calculated by an update scheme via 'initializing' $\hat{\theta}^1_j$ as 0 and adding $K^\wup_j(\gvar^\ith_{j-1}, \zetaj^\ith)$ after every backward-backward pass. This way, only the accumulated $\hat{\theta}^1_j$ needs to be kept in memory, compared to every summand.
    \item $K_j^\wup(\ph,x_{j-1})$ is linear, such that $$\Sigma_i K_j^\wup(\eta_j^\ith,x_{j-1}) =  K_j^\wup(\Sigma_i\eta_j^\ith,x_{j-1}),$$ which means that we can accumulate $\hat{\eta}_j=\Sigma_i \eta_j^\ith$.
    \item $\hat{\eta}_L$ can be calculated (similarly to $\hat{\theta}^1_j$) via initialization as 0 and updating after every backward-backward pass by adding $\hat{\eta}^\ith_L$.
    \item Since $G''_j(z_j)$ is the null operator, $\eta_j^\ith$ depends \emph{linearly} on $\gamma_j^\ith$. As a result, $$\hat{\eta}_j = G'_j(z_j) \cdot \left( \Sigma_i \gamma_i \right)=G'_j(z_j) \cdot \hat{\gamma}_j$$ with $\hat{\gamma}_j = \Sigma_i \gamma_j^\ith$.
    \item $\hat{\eta}_j$ and consequently $\hat{\theta}^2_j=K_j^\wup(\hat{\eta}_j,x_{j-1})$ are linear in $\hat{\eta}_L$ and can be calculated recursively from $\hat{\eta}_L$. This way, only one forward-backward pass needs to be performed, compared to the $C$ passes that normally need to be performed.
    \item By erasing variables from memory once they are no longer needed, this optimized algorithm does not require more memory than a single conventional double backpropagation procedure.
\end{enumerate}
We end up with $2L-1 + 2CL$ linear operations ($L$ for the forward pass, $L$ for each of the $C$ backward and backward-backward passes and $L-1$ for the forward-backward pass). As the na\"{i}ve implementation requires $L+3CL-C$ linear operations, about a third of the linear operations are saved (because $2CL$ respectively $3CL$ represent the bulk of the operations).
The algorithm is presented in detail in Algorithm \ref{alg:frobenius_jacobian_optimized} in Appendix \ref{sec:appdx_algorithm}.

\section{Loss Landscapes for (leaky) ReLU networks}
\begin{figure*}[!t]
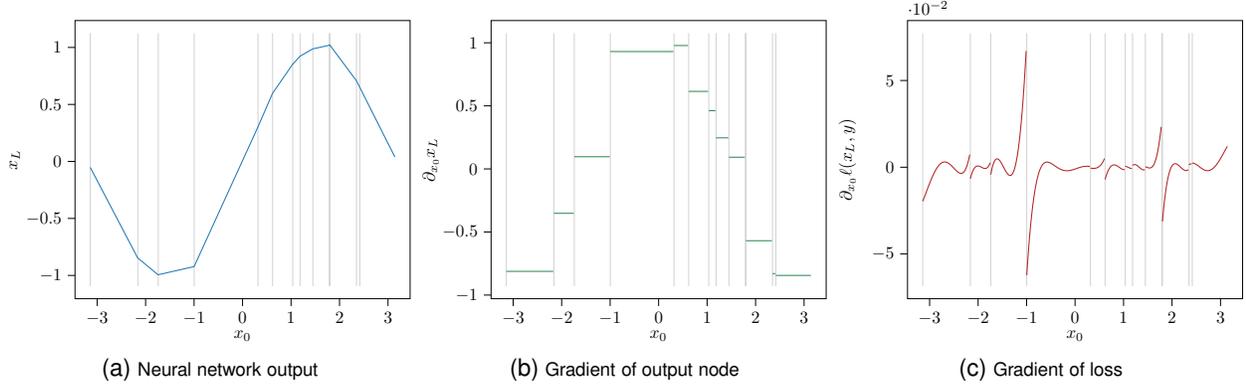

\centering
\subfloat[\scriptsize Neural network output]{\input{images/approximation}%
\label{fig:approximation}}
\subfloat[\scriptsize Gradient of output node]{\input{images/gradients}%
\label{fig:gradients}}
\subfloat[\scriptsize Gradient of loss]{\input{images/db_gradients}%
\label{fig:db_gradients}}
\caption{Neural network trained to approximate a sine-curve. Vertical grey lines symbolize the boundaries of locally affine regions.}
\label{fig_sim}
\end{figure*}
In the following, we will only consider finite-dimensional networks. The loss landscapes of (leaky) ReLU networks represent special cases due to jump discontinuities in their derivatives.
\subsection{Loss Landscape in the Inputs}
(Leaky) ReLU networks partition the input space into convex polytopes in which the logit layer $z_L$ is affine in $x_0$ \citep{raghu2017expressive}. This in turn means that the operator
\begin{equation}\label{eq:end_operator}
    \D{x_0}{z_L}
\end{equation}
is constant in $x_0$ in the interior of each convex polytope and in turn locally constant almost everywhere\footnote{A technical remark: The points of non-differentiability lie on the boundary of these convex polytopes. There are extensions of (Fr\'{e}chet) differentiability such as different types of \emph{subgradients and -differentials} that may apply to points on the boundary. See \citep{mordukhovich2006variational} for an overview of the many different concepts of subgradients. A full theoretical subdifferential treatment however is far out of the scope of this paper. Since the boundaries of the polytopes form a null set, we do not consider these points here.}. Since the penalty term is given by $$\R=p\left( \left( \D{x_0}{z_L}\right)^\ast\cdot \zeta_L \right),$$ this implies that $\R$ is locally constant in $x_0$ within each convex polytope if (and only if) $\zeta_L$ is locally constant in this region as well. Nevertheless, a jump discontinuity may occur when for some $z_j$, the activation $x_j=g_j(z_j)$ enters a different locally linear region of the (leaky) ReLU nonlinearity. This is the case when an entry of the vector $z_j$ switches between $(-\infty,0]$ and $(0,\infty)$.

\subsection{Loss Landscape in the Parameters}\label{sec:landscape_parameters}
The above considerations lead to the question, whether $\R$ may also be locally constant almost everywhere in the parameter space. If that were the case, any derivative-based optimization algorithm (like stochastic gradient descent) would instantly fail, because then automatically $\nabla_{\Theta}{\R}= 0$. However, for fixed $x_0$ in the interior of a convex polytope, the operator \eqref{eq:end_operator} is locally affine in the linear weights $\theta_j$ and locally constant in the biases $b_j$.  As $z_j$ depends (locally affine) on $\theta_k$ and (locally constant) on $b_k$ (for $k\leq j$), this means that $\R$ is luckily not locally constant almost everywhere in $\Theta$. The exact functional dependence then hinges on whether and how $\zeta_L$ depends on the weights.\newline
However, $\R$ thus also 'inherits' the jump discontinuities from $\D{x_0}{z_L}$, which may introduce numerical problems when using derivative-based optimization.\\

In reality, the problem of jump discontinuities may however be not as severe as it may seem at first glance: Usually, the optimization methods are applied not on a penalty term for a single point $x_0$ (with label $y$), but on the average value of $\R$ for a whole batch $\{(x_0^{(i)},y^{(i)})\}_{i=1,\dots,M}$. While the number of jump discontinuities adds up over the number of samples in this batch, the averaging process introduces a 'smoothing' effect on the loss landscape. These phenomena are empirically demonstrated on a simple toy example.

\subsection{Experiments}
For the following extremely simple toy example, we created a dataset of 1500 points $\lbrace (x_0^{(i)},y^{(i)})\rbrace_{i=1,\dots,1500}$, where $x_0^{(i)}\in\left[ -\pi, \pi \right]$ and $y^{(i)}=\sin(x^{(i)})$. We then fitted a small multilayer perceptron with 2 hidden ReLU layers (with 8 respectively 5 neurons) and a linear output layer (with 1 neuron and $g_L=\text{id}$) to this dataset, using the squared loss.\\
We start by considering the loss landscape in the inputs. In Figure \ref{fig:approximation}, the resulting approximating neural network is compared to the actual sine-curve. As expected, the neural network creates a locally affine, continuous output. Since the network maps real numbers to real numbers, we can identify the operator \eqref{eq:end_operator} with the partial derivative $\partial_{x_0} x_L=\partial_{x_0} z_L \in \mathbb{R}$. As displayed in Figure \ref{fig:gradients}, this derivative exhibits locally constant regions separated by the locations of non-differentiability. As a consequence, any penalty term $$\R_\text{node} := p \left(\partial_{x_0} z_L\right)$$ for some a.e. differentiable $p$ would necessarily be locally constant in $x_0$ as well (not depicted here). The loss landscape in $x_0$ for the classical double backpropagation penalty $$\R_\text{cdb} := \norm{\partial_{x_0} \ell (x_L,y)}_2^2 = \left(\left(\partial_{x_0} (x_L-y)^2\right)\right)^2$$
is depicted in Figure \ref{fig:db_gradients} and shows the expected jump discontinuities.


Since even a neural network as small as this one has 61 parameters, one cannot feasibly depict the loss landscape over \emph{all} parameters. Therefore, we fix the weights of the trained network and vary only one parameter of each the weight matrix and bias vector of the second hidden layer. We will call these parameters $w$ and $b$. In Figure \ref{fig:single_gradient_wb}, for fixed $x_0\approx 1.022$, the dependence of $s:=\partial_{x_0} z_L \in \mathbb{R}$ on $w$ respectively $b$ is shown. As predicted in section \ref{sec:landscape_parameters}, $s$ is locally affine in $w$ and exhibits jump discontinuities. Furthermore, $s$ as a function of $b$ is locally constant and exhibits jump discontinuities.\\

\begin{figure}
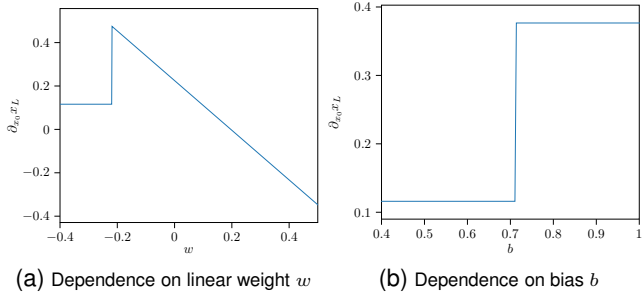

\centering
\subfloat[\scriptsize Dependence on linear weight $w$]{\input{images/single_gradient_w.tex}%
\label{fig:single_gradient_w}}
\subfloat[\scriptsize Dependence on bias $b$]{\input{images/single_gradient_b.tex}%
\label{fig:single_gradient_b}}
\caption{Dependence of $\partial_{x_0} z_L$ on $w$ and $b$.}\label{fig:single_gradient_wb}
\end{figure}

For the actual optimization, the properties of interests are the derivatives of the \emph{penalty terms}. For $\R_\text{node}$, we choose $p: s \mapsto s^2$ and visualize $\partial_w \R_\text{node}$ and $\partial_b \R_\text{node}$ in Figure \ref{fig:single_dRnode_dwb}. While $\partial_w \R_\text{node}$ exhibits a piecewise linear behavior (including a locally constant portion) with a jump discontinuity, $\partial_b \R_\text{node}$ is constant 0 (as a consequence of $\R_\text{node}$ being locally constant due to $g_L=\text{id}$, as explained in section \ref{sec:landscape_parameters}). This demonstrates how first-order optimization of $\R_\text{node}$ for a single example $x_0$ may suffer from instabilities, whenever a neuron switches between the locally linear regions of the (leaky) ReLU nonlinearity.\newline
We perform a similar analysis for the classical backpropagation penalty $\R_\text{cdb}$ and display our results in Figure \ref{fig:single_dRcdb_dwb}. While the jump discontinuities appear in the same spots, the non-constant portion of $\partial_w \R_\text{cdb}$ exhibits nonlinear behavior due to the (nonlinear) choice of $p$ and the dependence of $\xi_0$ on $w$. A central difference in the behavior of the bias derivative $\partial_b \R_\text{cdb}$ compared to $\partial_b \R_\text{node}$ lies in the fact that this derivative is \emph{not} constant 0. This is because classical double backpropagation in general yields $\eta_L \neq 0$.\\
\begin{figure}
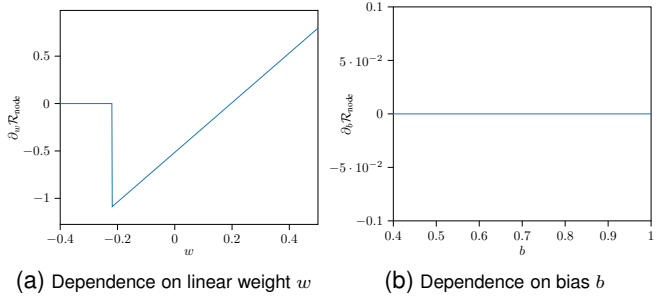

\centering
\subfloat[\scriptsize Dependence on linear weight $w$]{\input{images/single_dRnode_dw.tex}%
\label{fig:single_dRnode_dw}}
\subfloat[\scriptsize Dependence on bias $b$]{\input{images/single_dRnode_db.tex}%
\label{fig:single_dRnode_db}}
\caption{Derivatives of the penalty term $\R_\text{node}$ with respect to $w$ and $b$.}\label{fig:single_dRnode_dwb}
\end{figure}
\begin{figure}
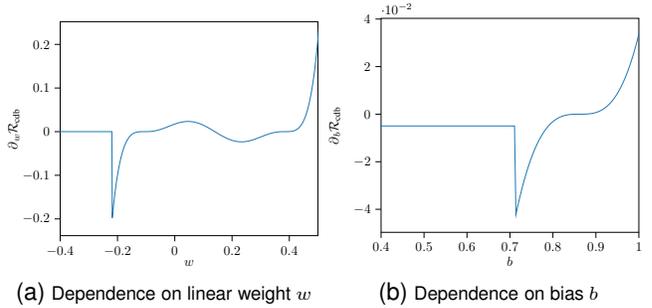

\centering
\subfloat[\scriptsize Dependence on linear weight $w$]{\input{images/single_dRcdb_dw.tex}%
\label{fig:single_dRcdb_dw}}
\subfloat[\scriptsize Dependence on bias $b$]{\input{images/single_dRcdb_db.tex}%
\label{fig:single_dRcdb_db}}
\caption{Derivatives of the penalty term $\R_\text{cdb}$ on $w$ and $b$.}\label{fig:single_dRcdb_dwb}
\end{figure}

As we will show now, the feared instabilities can be reduced by batchwise optimization, which is standard practice. To visualize this, we randomly pick a batch $\mathfrak{B}=\lbrace(x_0^{(i)},y^{(i)}))\rbrace_{i=1,\dots,M}$ for $x^{(i)}\in \left[-\pi, \pi \right]$, $y^{(i)}=\sin (x^{(i)})$ with batch size $M=256$ and visualize the averaged penalty terms $$\R^\mathfrak{B}_\text{node} := \frac{1}{M} \sum\limits_{i=1}^{M} \left( \partial_{x_0^{(i)}} x_L^{(i)} \right)^2$$
and
$$\R^\mathfrak{B}_\text{cdb} := \frac{1}{M} \sum\limits_{i=1}^{M} \norm{\partial_{x^{(i)}_0} \ell (x^{(i)}_L, y^{(i)})}_2^2$$
in Figures \ref{fig:averaged_dRnode_dwb} and \ref{fig:averaged_dRcdb_dwb}.
While jump discontinuities of the averaged penalty terms are still visible, the fact that the individual discontinuities lie close together in parameter space creates the effect of 'almost smooth' loss landscapes. Due to this smoothing effect, the optimization using batch optimization is much less impaired by the discontinuities than for a single example $x_0$, which explains their success in the applications listed in section \ref{sec:examples}, even when using (leaky) ReLU activation functions.

\begin{figure}[t]
\centering
\subfloat[\scriptsize Dependence on linear weight $w$]{\input{images/averaged_dRnode_dw.tex}%
\label{fig:averaged_dRnode_dw}}
\subfloat[\scriptsize Dependence on bias $b$]{\input{images/averaged_dRnode_db.tex}%
\label{fig:averaged_dRnode_db}}
\caption{Derivatives of the penalty term $\R^\mathfrak{B}_\text{node}$ on $w$ and $b$. The averaging over the batch $\mathfrak{B}$ creates a 'smoothed' landscape over $w$.}\label{fig:averaged_dRnode_dwb}
\end{figure}

\begin{figure}
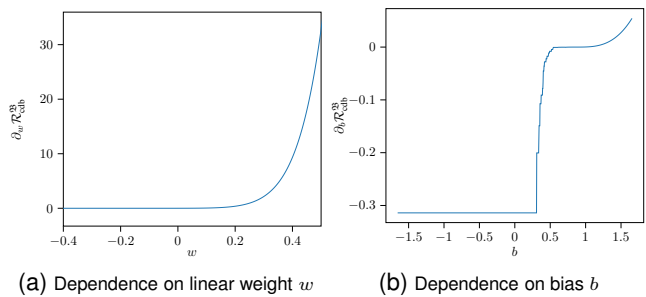

\centering
\subfloat[\scriptsize Dependence on linear weight $w$]{\input{images/averaged_dRcdb_dw.tex}%
\label{fig:averaged_dRcdb_dw}}
\subfloat[\scriptsize Dependence on bias $b$]{\input{images/averaged_dRcdb_db.tex}%
\label{fig:averaged_dRcdb_db}}
\caption{Derivatives of the penalty term $\R^\mathfrak{B}_\text{cdb}$ on $w$ and $b$. The averaging over the batch $\mathfrak{B}$ creates a 'smoothed' landscape.}\label{fig:averaged_dRcdb_dwb}
\end{figure}

\section{Conclusion}
In this paper, we provided a first in-depth description of 'double backpropagation' procedures, which come into play whenever a loss function contains derivatives of output nodes with respect to input nodes. We offer a unified perspective for a large class of such loss functions and describe the derivatives in the general framework of Fr\'{e}chet derivatives on Hilbert spaces. For this, we developed a theory of adjoint operators for continuous, bilinear operators, which covers many common layer types. The obtained description of the involved derivatives allows us to present optimized double backpropagation schemes for some networks, which reduces the time complexity by roughly a third in this case. Furthermore, we provided a description for the (discontinuous) loss landscape for derivative-based losses of (leaky) ReLU networks both in the inputs as well as the parameters. We further demonstrate that training in batches introduces a 'pseudo-smoothing' effect to the loss landscape, which results in higher numerical stability of the training procedure.

\bibliographystyle{plainnat}
\bibliography{referencesDB}
\newpage
\mbox{~}
\clearpage
\appendix
\section*{Appendix}
\section{Proofs for Bilinear Operator Theory}\label{sec:appdx_bilinear}
In order for our proofs to work, we need a very simple corollary.
\begin{corollary}\label{col:zero-property}
Let $U,V,W$ be real Hilbert-spaces. Let further $A: U \times \mathcal{V} \rightarrow \mathcal{W}$ and $B: U \times \mathcal{V} \rightarrow \mathcal{W}$ be continuous bilinear operators. Then $A=B$ if and only if
$$\langle A(u,v),w \rangle_\mathcal{W} = \langle B(u,v),w \rangle_\mathcal{W}$$
for all $(u,v,w) \in U\times\mathcal{V}\times\mathcal{W}$. 
\end{corollary}
\begin{proof}
The '$\Rightarrow$'-direction follows immediately. We now prove the converse direction.\newline
Let $$\langle A(u,v),w \rangle_\mathcal{W} = \langle B(u,v),w \rangle_\mathcal{W}$$
for all $(u,v,w) \in U\times\mathcal{V}\times\mathcal{W}$. Then
$$\langle \left[ A - B \right] (u,v),w \rangle_\mathcal{W} = 0,$$
in particular for $w = \left[ A - B \right] (u,v)$, so that 
$$\langle \left[ A - B \right] (u,v), \left[ A - B \right] (u,v) \rangle_\mathcal{W}=0,$$
which means that (due to the definiteness of inner products), $$\left[ A - B \right] (u,v)=0$$ for all $u\in U$ and $v \in \mathcal{V}$. In other words, $A-B$ is the null operator, or $A=B$.
\end{proof}

We can now prove the following lemma:
\begin{lemma}
$K^\transp$ and $K^\wup$ are bilinear operators.
\end{lemma}

\begin{proof}
We now show the proof for the bilinearity of $K^\transp$ only. The proof for the bilinearity of $K^\wup$ is completely analogous.\newline
We already know that $K^\transp(\theta,y)$ is linear in $y$, since it is defined to be the adjoint operator of $K(\theta,\cdot)$, which is always a linear operator itself. It remains to show that $K^\transp(\theta,y)$ is also linear in $\theta$:
\begin{equation}\begin{aligned}
\langle x, K^\transp(\theta+\psi,y) \rangle_\mathcal{X} &= \langle K(\theta+\psi,x),y \rangle_\mathcal{Y} \\
&= \langle K(\theta,x),y \rangle_\mathcal{Y} + \langle K(\psi,x),y \rangle_\mathcal{Y} \\
&= \langle x,K^\transp(\theta,y) \rangle_\mathcal{X} + \langle x,K^\transp(\psi,y) \rangle_\mathcal{X}
\end{aligned}\end{equation}
It follows that
\begin{equation}\begin{aligned}
\langle x, K^\transp(\theta+\psi,y) - \left[ K^\transp(\theta,y) + K^\transp(\psi,y) \right] \rangle_\mathcal{X} = 0 
\end{aligned}\end{equation}
for all $x,y,\theta,\psi$, which, according to Corollary \ref{col:zero-property}, shows that
\begin{equation}\begin{aligned}
K^\transp(\theta+\psi,y) = K^\transp(\theta,y) + K^\transp(\psi,y),
\end{aligned}\end{equation}
meaning that $K^\transp$ is linear in the first argument and thus a bilinear operator.
\end{proof}

\adjointTheorem*
\begin{proof}
Due to the continuity in all arguments, which are elements of Hilbert spaces, the adjoints exist. Let $K^\otimes$ be the adjoint of $K^\transp(\theta,y)$ in $\theta$ and let $K^\oplus$ be the adjoint of $K^\wup(x,y)$ in $x$. Then
\begin{equation}\begin{aligned}
	\langle K(\theta,x),y\rangle_\Y &= \langle \theta, K^\wup(x,y) \rangle_\P = \langle K^\oplus(\theta,y),x\rangle_\X\\
	&= \langle x, K^\transp(\theta,y) \rangle_\X = \langle K^\otimes(x,y),\theta\rangle_\P\\
\end{aligned}\end{equation}
for all $(\theta,x,y) \in \P\times\X\times\Y$.
From Corollary \ref{col:zero-property}, we can infer that $K^\transp=K^\oplus$ and $K^\wup=K^\otimes$.
\end{proof}

\begin{remark}
Theorem \ref{thm:adjoint-commutativity} immediately generalizes to multilinear operators. For Hilbert spaces $\X_1,\dots,\X_{n+1}$, which we uniquely identify by their indices\footnote{This becomes important when $\X_i=\X_j$ for some $(i,j)$.}, let
$$K: \X_1 \times \cdots \times \X_{n} \rightarrow \X_{n+1}$$ be a multilinear, continuous operator and let $$K^{\X_i}: \X_1 \times \cdots \times \X_{i-1} \times \X_{i+1} \times \cdots \times \X_{n+1} \rightarrow \X_i$$ be the adjoint of $K$ in its $i$-th argument, which is a multilinear operator. If $K^{\X_i}$ is continuous, then $$\left(K^{\X_i}\right)^{\X_j}=K^{\X_j}$$ for all $j$. Thus, if $K^{\X_i}$ is continuous for all $i$, this means that only the final argument with respect to which one takes the adjoint determines the resulting operator.
\end{remark}

\symmetricOpsSelfAdjoint*
\begin{proof}
First of, since $M(x,\ph)=M(\ph,x)$ for all $x$, we have $$M^\transp(x,\ph)=(M(x,\ph))^\ast=(M(\ph,x))^\ast=M^\wup(\ph,x)$$ for all $x\in\X$. It follows that $M^\wup=M^\transp$, which is thus also a symmetric bilinear operator.
With this we have \begin{equation}
    \begin{aligned}
        \langle M(x,y),z \rangle_\X &=\langle y,M^\transp (x,z)\rangle_\X = \langle y,M^\transp (z,x) \rangle_\X\\
        &=\langle M(z,y),x \rangle_\X=\langle z,M^\transp(x,y) \rangle_\X,
    \end{aligned}
\end{equation} which implies that $M=M^\transp=M^\wup$ with Corollary \ref{col:zero-property}.
\end{proof}

\multiplicationOpSelfAdjoint*
\begin{proof}
For all $x,y\in\X$,
 $$\langle Ax,y \rangle_\X = \langle M(a,x),y \rangle_\X = \langle x,M(a,y) \rangle_\X$$ according to Lemma \ref{lem:symmetric_ops_self_adjoint}. Furthermore, $\langle Ax,y \rangle_\X = \langle x,A^\ast y \rangle_\X$, so that $\langle x,A^\ast y \rangle_\X=\langle x,M(a,y) \rangle_\X$ for all $x,y \in \X$, which means that $A^\ast=M(a,\ph)=A$ due to Corollary \ref{col:zero-property}.
\end{proof}

\section{Initial Values $\eta_L$}\label{sec:appdx_initial_etaL}
In section \ref{sec:penalty_terms}, a generalization of the different penalty terms to the form $$\R:=p\left( \left(\D{x_0}{x_L} \right)^\ast \cdot v \right)$$ was introduced. The initial value $\eta_L$, which is needed in order to initialize Algorithm \ref{alg:weight_gradients}, depends on the out layer's activation function $g_L$ and $v=\xi_L$. Typical special cases for the activation function include $g_L=\text{softmax}$ (for classification problems) or $g_L=\text{id}$ (non-categorical targets like in regression or if one wants to apply penalties to derivatives of logits). Because $\text{softmax}$ is not a coordinate-wise activation function, we cannot harness equation \eqref{eq:D-zeta-theta}. \\
For $v$, we can identify two particular special cases: Those where $v$ is independent of the network and $v=\nabla_{x_L}\L=-y\oslash x_L$ for classical double backpropagation.\\
For these reasons, $\eta_L$ needs to be calculated explicitly for the cases above.

\subsection{Softmax}
Here, we derive $\eta_L$ for penalty terms of the form \eqref{eq:penalty_general_form}, where $v$ is independent of the neural network's input and $g_L=\text{softmax}$. A standard calculation shows that
$$\D{z_L}{x_L}\simeq \text{diag}(x_L) - x_L x_L^T$$
is self-adjoint (symmetric). We hence have
\begin{equation}
    \begin{aligned}
    \zeta_L=& \left( \D{z_L}{x_L} \right)^\ast \cdot v\\
    =& \hphantom{.} x_L \odot v - \langle x_L, v \rangle \cdot x_L
    \end{aligned}
\end{equation}
and particular, for $v=e^\ith$ and $x_L=(x_L^1,\dots,x_L^C)^T$, this can be written as
\begin{equation}
    \begin{aligned}
    \left( \D{z_L}{x_L} \right)^\ast \cdot e^\ith &= x^i_L e^\ith - x^i_L \cdot x_L \\
    &= x_L^i \cdot (e^\ith - x_L).
    \end{aligned}
\end{equation}
Since $v$ does not depend on $x_L$, we can treat $v$ as a constant, which yields $$\D{x_L}{\zeta_L} \simeq \text{diag}(v) - \langle x_L,v  \rangle\cdot I - v x_L^T$$
(with identity matrix $I$), leading to
\begin{equation*}
    \begin{aligned}
    \eta_L &= \left( \D{z_L}{x_L} \right)^\ast\left( \D{x_L}{\zeta_L} \right)^\ast \cdot \hh_L\\
    &= \left( \D{z_L}{x_L} \right)^\ast \left(\text{diag}(v) - \langle x_L,v  \rangle\cdot I - v x_L^T \right)^T \cdot \hh_L \\
    &= \left( \D{z_L}{x_L} \right)^\ast \left(v \odot x_L - \langle x_L,v  \rangle \cdot \hh_L - \langle x_L, \hh_L \rangle \cdot x_L  \right) \\
    &= x_L \odot v \odot h_L - \langle  x_L,v\rangle(x_L\odot v) - \langle x_L, v\odot h_L \rangle  \\
    &\hphantom{=} + 2 \langle x_L,v \rangle \langle x_L,h_L\rangle x_L.
    \end{aligned}
\end{equation*}

\subsection{Softmax + Non-Negative Log-Likelihood Loss}
The following deals with classical double backpropagation, where a penalty of the form $$p(\nabla_{x_0} \ell (x_L,y))$$ is applied, with $\ell$ denoting the non-negative log likelihood loss function as defined in equation \eqref{eq:NLL}. As detailed in section \ref{sec:classical_dbp}, this penalty term can be written in the general form \eqref{eq:penalty_general_form} with $g_L=\text{softmax}$ and $v=\nabla_{x_L}\ell(x_L,y)= -y\oslash x_L$.
\begin{equation}
    \begin{aligned}
    \left( \D{z_L}{x_L} \right)^\ast \cdot \nabla_{x_L}\L &= \left( \text{diag}(x_L) - x_L x_L^T\right) \cdot \left(-y \oslash x_L \right)\\
    &= -y + \left(\sum_{i=1}^C (x_L^iy^i)/x_l^i \right) x_L \\
    &= x_L - y,
    \end{aligned}
\end{equation}
where we used that $\sum_i y^i = 1$, where $y=(y^1,\dots,y^C)^T$. 
And for classic double backpropagation:
\begin{equation}
    \begin{aligned}
    \eta_L &= \left( \D{z_L}{x_L} \right)^\ast\cdot \left( \D{x_L}{\zeta_L} \right)^\ast \cdot \hh_L\\
    &= \left( \text{diag}(x_L) - x_L x_L^T \right)\cdot \left(\text{id} \right)^\ast \cdot \hh_L \\
    &=  x_L \odot \hh_L - x_L \langle x_L,\hh_L \rangle
    \end{aligned}
\end{equation}

\subsection{Identity Function}
If one applies a penalty to the derivatives of the logit-layers ($z_L$), one can still represent this case as in equation \eqref{eq:penalty_general_form} by modelling $g_L$ as an identity map, so that $$\D{z_L}{x_L}=\text{id}$$ and $\zeta_L = \xi_L = v$. Since $\zeta_L$ is constant in $z_L$, we have $\eta_L = \left( \D{z_L}{\zeta_L} \right)^\ast \cdot \hh_L = 0$.

\section{Optimized Algorithm for Jacobian Penalties}\label{sec:appdx_algorithm}
\begin{algorithm}[H]
  \caption{For a neural network with locally linear activation functions (e.g. ReLU) and softmax output, the weight gradients of the penalty term $\R=\sum_{i=1}^C\R^\ith=\sum_{i=1}^C \|x^i_L\|_{\X_0}^2$ can be calculated using an algorithm with only $2CL+2L-1$ linear operations, compared to the na\"{i}ve implementation with $3CL+L-C$. Additionally, the memory requirements in $\mathcal{O}(1)$ in $C$.} \label{alg:frobenius_jacobian_optimized}
\begin{algorithmic}
   \STATE \# forward pass
   \STATE {\bfseries Initialize} $x_0$
   \FOR{$j\gets 1$ {\bfseries to}  $L$}
   \STATE $z_j = K_j(\thetaj,x_{j-1}) + b_j \hspace{2mm}$
   \STATE $x_j = g_j(\zj)$
   \STATE $a_j = g'_j(z_j)$, $\Delete{z_j}$
   \ENDFOR

   \STATE {\bfseries Initialize} $\hat{\theta}^1_j = 0 $, $\hat{\eta}_j=0$
   \FOR{$i\gets 1$ {\bfseries to}  $C$}

      \STATE \# $i$-th backward pass
      \STATE {\bfseries Initialize} $\xi^{\langle i \rangle}_L = e^\ith$
      \FOR{$j\gets L$ {\bfseries to}  $1$}
      \STATE $\zeta_j^{\langle i \rangle} = M_j (a_j, \xi^{\langle i \rangle}_j) $, $\Delete{\xi^{\langle i \rangle}_j}$
      \STATE $\xi^{\langle i \rangle}_{j-1} = K_j^T(\theta_j,\zeta^{\langle i \rangle})$
      \ENDFOR

      \STATE \# $i$-th backward-backward pass
      \STATE {\bfseries Initialize} $\gvar_0^\ith = \nabla_{\xi_0^\ith} \|\xi_0^\ith\|^2_{\X_0} = 2\xi_0^\ith$
      \FOR{$j\gets 1$ {\bfseries to}  $L$}
      \STATE $\hat{\theta}^1_j \leftarrow  \hat{\theta}_j + K^\wup_j(\gvar_{j-1}^\ith,\zeta_{j}^\ith)$, $\Delete{\zeta_j^\ith}$
      \STATE $\hh_j^\ith=K_j(\theta_j,\gvar^\ith_{j-1})$, $\Delete{\gvar^\ith_{j-1}}$
      \STATE $\gvar_j^\ith = M_j (a_j, \hh_j^\ith) $
      \STATE \textbf{if} $j\neq L$ \textbf{then} $\Delete{\hh_j^\ith}$ \textbf{end if}
      
      \ENDFOR
      
      \STATE $\hat{\eta}_L \leftarrow \hat{\eta}_L + x_L \odot e^\ith \odot h^\ith_L - \langle  x_L,e^\ith\rangle(x_L\odot e^\ith)$ 
      \STATE $\hat{\eta}_L \leftarrow \hat{\eta}_L + 2 \langle x_L,e^\ith \rangle \langle x_L,h^\ith_L\rangle x_L - \langle x_L, e^\ith\odot h^\ith_L \rangle$
      \STATE $\Delete{h^\ith_L, x_L}$

   \ENDFOR

      \STATE \# cumulated forward-backward passes
      \FOR{$j\gets L$ {\bfseries to}  $1$}
      \STATE $\nabla_{\theta_j}\R = K^\wup_j(\gvar_{j-1},    \zetaj) + K^\wup_j\left(\eta_j,x_{j-1}\right)$
      \STATE $\nabla_{b_j}\R = \eta_j$
      
      \IF{$j>1$}
      \STATE $\hat{\gamma}_{j-1} = K_j^T(\theta_j, \hat{\eta}_j)$
      \STATE $\Delete{\hat{\eta}_{j}}$
      \STATE $\hat{\eta}_{j-1} = M_{j-1}(a_{j-1},\hat{\gamma}_{j-1})$
      \STATE $\Delete{a_{j-1},\hat{\gamma}_{j-1}}$
      \ENDIF
      \ENDFOR
\end{algorithmic}
\end{algorithm}

\end{document}